\documentclass[twoside]{article}

\usepackage[accepted]{aistats2019}
\usepackage[round]{natbib}

\usepackage{color,soul}
\usepackage{amsmath}
\usepackage{amsthm}
\usepackage{amstext}
\usepackage{mathrsfs}
\usepackage{amssymb}
\usepackage{amsfonts}
\usepackage{bm}
\usepackage{blkarray}
\usepackage[utf8]{inputenc}
\usepackage{chngcntr}
\usepackage{apptools}
\AtAppendix{\counterwithin{thm}{section}}
\usepackage{amsmath,tikz}
\usetikzlibrary{matrix}

\newcommand{\overbar}[1]{\mkern 1.5mu\overline{\mkern-1.5mu#1\mkern-1.5mu}\mkern 1.5mu}

\newtheorem{thm}{Theorem}
\newtheorem{lemma}[thm]{Lemma}

\newtheorem{cor}[thm]{Corollary}
\newtheorem{prop}[thm]{Proposition}

\theoremstyle{definition}
\newtheorem{definition}[thm]{Definition}

\newtheorem{remark}[thm]{Remark}
\newtheorem{example}[thm]{Example}

\newcommand{\fun}[2]{#1\!\left(#2\right)}

\newcommand{\PP}[2]{\fun{P_{#1\!}}{#2}} 

\newcommand{\given}{\vert}

\newcommand{\size}[1]{|#1|}

\newcommand{\data}{d} 
\newcommand{\dataSpace}{\mathcal D} 
\newcommand{\dataSize}{\size{\dataSpace}}
\newcommand{\concept}{h} 
\newcommand{\conceptSpace}{\mathcal H}
\newcommand{\conceptSize}{\size{\conceptSpace}}
 
\newcommand{\CI}{\mathrm{CI}} 
\newcommand{\CR}{\mathrm{CR}}

\newcommand{\sumxi}{\sum_{i=1}^{\size{\dataSpace}}}
\newcommand{\sumhj}{\sum_{j=1}^{\size{\conceptSpace}}}

\newcommand{\LL}{\PP{L}{\concept\given d}}
\newcommand{\LLpri}{\PP{L_0}{\concept}}
\newcommand{\LLmar}{\PP{L}{d}}

\newcommand{\TT}{\PP{T}{d\given\concept}} 
\newcommand{\TTpri}{\PP{T_0}{d}}
\newcommand{\TTmar}{\PP{T}{\concept}}

\newcommand{\LLmat}{\mathbf{L}}

\newcommand{\TTmat}{\mathbf{T}}

\newcommand{\Likmat}{\mathbf{M}\,}

\newcommand{\M}{\mathbf{M}}

\newcommand{\A}{\mathcal{A}} 

\newcommand{\B}{\mathcal{B}}
\newcommand{\Pn}{\Phi_n}
\newcommand{\R}{\mathcal{R}}
\newcommand{\diag}{\mathrm{diag}}
\newcommand{\e}{\epsilon}
\newcommand{\crs}{\overset{cr}{\sim}}

\begin{document}

%

%

\twocolumn[

\aistatstitle{Generalizing the theory of cooperative inference}

\aistatsauthor{Pei Wang \And Pushpi Paranamana \And  Patrick Shafto}

\aistatsaddress{Department of Mathematics \& Computer Science, Rutgers University---Newark } ]

\begin{abstract}
Cooperative information sharing is important to theories of human learning and has potential implications for machine learning. 
Prior work derived conditions for achieving optimal Cooperative Inference given relatively restrictive assumptions. 
We demonstrate convergence for any discrete joint distribution, robustness through equivalence classes and stability under perturbation, and effectiveness by deriving bounds from structural properties of the original joint distribution. 
We provide geometric interpretations, connections to and implications for optimal transport and to importance sampling, and  conclude by outlining open questions and challenges to realizing the promise of Cooperative Inference. 
\end{abstract}

\section{Introduction}



Cooperative information sharing is fundamental to human learning and finds applications in machine learning. The core idea of cooperation in human inference stems from work by Grice (Gricean Maxims; Grice, 1975)\nocite{grice1975logic} in linguistic pragmatics. Recent work in the linguistics literature has formalized the Rational Speech Act model \citep{frank2012predicting,goodman2013knowledge,Kao2014,lassiter2017adjectival}, which builds on earlier models of cooperative information sharing \citep{shafto2008teaching,Shafto2014,shafto2012learning}.
Indeed, these phenomena are not limited to language. Related models have been proposed to explain infants' learning from parents \citep{tomasello2009we,csibra2009natural,bonawitz2011double,shafto2012learning,buchsbaum2011children,gweon2014sins,shneidman2016learning,Eaves2016c}, 
learning from cooperative teachers \citep{shafto2008teaching,Shafto2014}, and how people decide who to trust \citep{shafto2012epistemic,Eaves2016c}. 
This model is of interest for explaining human learning and communication, but it lacks overarching theory regarding when and why cooperation might facilitate learning and communication. Our paper is a step toward a general mathematical theory for this work. 

Cooperative information sharing is of recent interest in machine learning. Explainability of machine learning models has been formalized in the Cooperative Inference (CI) framework \citep{yang2017explainable,vongbayesian}. There have also been several recent papers on learning from demonstrations that leverage cooperation in similar models \citep{ho2016showing,ho2018effectively}. Finally, cooperative inverse reinforcement learning is explicitly centered around CI 
\citep{hadfield2016cooperative,fisac2017pragmatic}.

 
We make four contributions toward strengthening the overarching theory of Cooperative Inference. 
(a) 
Section~\ref{sec: rectangular} proves convergence of CI for any rectangular matrix, which ensures that CI applies to any discrete model. 
(b) Cross ratio equivalence analysis in Section~\ref{sec:CR_eq} shows that the space of possible joint distributions is reducible to only working with distributions that differ in cross ratio(s). This provides a natural geometric structure over possible machine learning models, which is a highly general and very interesting direction for future work detailed in Section~\ref{sec:connections}. 
(c) 
Section~\ref{sec:sensitivity analysis} proves stability under perturbation which ensures robustness of inference where agents' beliefs differ. This shows that CI has the possibility of being viable in practice. 
(d) 
Section~\ref{sec:bounds} provides general bounds on effectiveness that are derived from structural properties of the initial matrix, $\M$. 
Section~\ref{sec:connections} provides a geometric interpretation, and connects to optimal transport and importance sampling, and  Section~\ref{sec:conclusion} conlcudes.
\section{Overview and Background}
All matrices in this paper are understood to be real, non-negative and have no zero rows or zero columns.
Matrices are in uppercase and their elements are in the corresponding lowercase.

These matrices can be thought as joint distributions of models throughout. In more detail, let $\conceptSpace$ be a concept space and $\dataSpace$ be a data space. For a given matrix $\M$,
each column can be viewed as a concept in $\conceptSpace$ and each row 
can be viewed as a data in $\dataSpace$. Normalizing by dividing the sum of its entries, $\M$ can be turned into a conditional distribution over $\conceptSpace$ or $\dataSpace$.
 

In this paper, we study the cooperative communication between a teacher and a learner.
Here, cooperation means that the teacher's selection of data depends on what the learner is likely to infer and vice versa.
The idea of \textit{cooperative inference} was introduced in \citep{YangYGWVS18}. We now briefly review their work.
\begin{definition}\label{def:CI}
For a fixed \textit{concept space} $\conceptSpace$ and a \textit{data space} $\dataSpace$, 
let $\LLpri$ be the learner's prior of a \textit{concept} $\concept$ among $\conceptSpace$ and $\TTpri$ be the teacher's prior of selecting a \textit{data} $\data$ from $\dataSpace$. 
The teacher's posterior of selecting  $\data$ to convey $\concept$ is denoted by $\TT$ and the learner's posterior for $\concept$ given $\data$ is denoted by $\LL$. \textbf{Cooperative inference} is a system shown below:
\begin{subequations}
\begin{align}
\LL &= \frac{\TT \LLpri}{\LLmar},
\label{eq:L}\\
\TT &= \frac{\LL \TTpri}{\TTmar},
\label{eq:T}
\end{align}
\label{eq:LT}
\end{subequations}
\noindent where $\LLmar$ and $\TTmar$ are the normalizing constants.
\end{definition}

Assuming uniform prior, \citep{YangYGWVS18} showed that Equation~\eqref{eq:LT} can be solved using 
\textbf{Sinkhorn iteration}~(\textbf{SK} for short; \citep{Sinkhorn1967}). The solution (if it exists) depends only on the initial joint distribution matrix, $\Likmat_{\dataSize \times \conceptSize}$, which defines the consistency between data and concepts. 



\textbf{Sinkhorn iteration} is simply the repetition of row and column normalization of $\Likmat$.
Denote the matrices obtained at the $k^\mathrm{th}$ row and column iteration of \eqref{eq:LT} by $\LLmat^{k}$ and $\TTmat^{k}$, respectively. Let their limits (if they exist) be $\LLmat:=\lim_{k\to\infty}\LLmat^{k}$ and $\TTmat:=\lim_{k\to\infty}\TTmat^{k}$.


\begin{example}\label{eg:M}
Consider a joint distribution matrix 
$\tiny{\M=
\begin{blockarray}{ccc}
&h_1 & h_2   \\
\begin{block}{c (cc)}
 d_1 & 1 & 1 \\
 d_2 & 0 & 1 \\
\end{block}
\end{blockarray}}$, where $m_{ij}=1$ if $d_i$ is consistent with $h_j$ and $m_{ij}=0$ otherwise, for $i,j=1,2$. \
The SK iteration proceeds as the following:
row normalization of $\M$ outputs: $\tiny\LLmat^{1}= \begin{pmatrix}\frac{1}{2} & \frac{1}{2} \\ 0 & 1 \end{pmatrix}$,
column normalization of $\LLmat^{1}$ outputs: $\tiny\TTmat^{1}= \begin{pmatrix}1 & \frac{1}{3} \\ 0 & \frac{2}{3} \end{pmatrix}$.
Iteratively, $\tiny\LLmat^{k}= \begin{pmatrix}1-\frac{1}{2k} & \frac{1}{2k} \\ 0 & 1 \end{pmatrix}$,
$\tiny\TTmat^{k}= \begin{pmatrix}1 & \frac{1}{2k} \\ 0 & 1-\frac{1}{2k} \end{pmatrix}$, and the limits exist as $k\to \infty$: 
$\tiny\LLmat=\TTmat=\M^*=\begin{pmatrix}1 &0 \\ 0 & 1 \end{pmatrix}$. 

For this $\M$, a teacher and a leaner who reason \textit{independently} 
can not reliably convey $h_1$ using $\dataSpace$; $d_1$, the only data that is consistent with $h_1$ is also consistent with $h_2$.
However, a teacher and learner that assume \textit{cooperation} can perfectly convey $h_1$ using $d_1$; in the converged joint distribution $\M^*$, $d_1$ is consistent only with $h_1$.
Intuitively, a cooperative teacher will pick $d_2$ to teach $h_2$, because picking $d_1$ would cause confusion for the learner. Correspondingly, when receiving $d_1$, the cooperative leaner will reason that the teacher must intend to teach $h_1$, because otherwise he would pick $d_2$.  In fact, the teaching between the cooperative pair is optimal, $\CI(\M)=1$ (Definition~\ref{def:CIndex}).
\end{example}

The \textit{Cooperative index} quantifies the effectiveness of the cooperative communication.
It is the average probability that a concept in $\conceptSpace$ can be correctly inferred by a learner given the teacher's selection of data. 

\begin{definition}\label{def:CIndex}
Given $\Likmat$ and assuming that SK iteration of \eqref{eq:LT} converges to a pair of matrices $\LLmat=(l_{ij})$ and $\TTmat=(t_{ij})$, we define the \textbf{cooperative index} as 
\[ \CI(\Likmat) = \frac{1}{\conceptSize} \LLmat \odot \TTmat=\frac{1}{\conceptSize} \sumhj \sumxi l_{i,j} t_{i,j}.\]
\end{definition}
Here, $\LLmat \odot \TTmat$ means the inner product between $\LLmat $ and $ \TTmat$.
The definition implies that $\CI(\Likmat)$ is invariant under row and column permutations of $\M$. 

Next we define a few useful technical terms.
\begin{definition}\label{def:positive_diag}
Let $A=(a_{ij})$ be an $n\times n $ matrix and $S_n$ be the set of all permutations of $\{1, 2, \dots, n\}$.
For any $\sigma \in S_n$, the set of $n$-elements $\{a_{1\sigma(1)},\dots, a_{n,\sigma(n)} \}$ is called a \textbf{diagonal}
of $A$. If every $a_{k\sigma(k)}>0$, we say that the diagonal is \textbf{positive}. 
An element $a_{i_0j_0}$ of $A$ is called \textbf{on-diagonal} if it is contained in a positive diagonal, otherwise $a_{i_0j_0}$ is called \textbf{off-diagonal}. In particular, $A$ may have a positive off-diagonal element. 
We use $\overbar{A}$ to denote the matrix obtained from $A$ by setting all its off-diagonal elements into zeros.
If $A$ contains no positive off-diagonal element, i.e $A=\overbar{A}$,  $A$ is said to have \textbf{total support}.
\end{definition}

\citep{YangYGWVS18} focused on the case when the data set and the hypotheses set have the same size. They showed that $0 \leq \CI(\Likmat)\leq 1$ for any $\M$ (if $\CI(\M)$ exists). In particular, when $\M$ is a square matrix, they showed that Equation~(\ref{eq:LT}) has a solution if and only if $\M$ has at least one diagonal and $\CI(\M)$ is optimal if and only if $\M$ has exactly one positive diagonal.
\section{Convergence of rectangular matrices}\label{sec: rectangular}
It is typical that the sizes of a data set and a concept set are different. Therefore, considering only square models is too restrictive. We show that the solution of Equation~\eqref{eq:LT} can be obtained using SK iteration for any rectangular joint distribution $\M$. This implies that cooperative inference can be performed on any discrete model.


First, we study the format of the limit of SK iteration on rectangular matrices. 
It is proven in \citep{sinkhorn1967concerning} that the limit (if exists) of SK iteration on a \textbf{square} $\M$ is a single doubly stochastic matrix $\M^*$, i.e. $\LLmat=\TTmat=\M^*$. As the numbers of rows and columns are different in a \textbf{rectangular} $\M$, the limit of the SK iteration on $\M$ is a pair of distinct matrices $(\LLmat, \TTmat)$, where, $\LLmat$ is row normalized and $\TTmat$ is  column normalized. Such a pair is called \textit{stable} defined below.


\begin{definition}\label{def:partial_pattern}
The \textbf{pattern} of a matrix $A$ is the set of entries where $a_{ij} > 0.$
Matrix $B$ is said to have a \textbf{partial pattern} of $A$, denoted by $B\prec A$, if $a_{ij}=0 \implies b_{ij}=0$.
\end{definition}

\begin{definition}\label{def:sk_stable}
A pair of $u\times v$-matrices $(P,Q)$ is called \textbf{stable} if column normalization of $P$ equals $Q$ and row normalization of $Q$ equals $P$. A matrix is \textit{stable} if it is contained in a \textit{stable} pair.  
\end{definition}
\begin{remark}
If $(P,Q)$ is \textit{stable}, then $P$ and $Q$ are row and column normalized, respectively. \textit{$SK$ iteration} of $P$ (or $Q$) results a sequence alternating between $P$ and $Q$. Moreover, $P$ and $Q$ must have the same pattern.
\end{remark} 

As mentioned above, the limit of SK iteration is doubly stochastic for a square $\M$. 
The following proposition provides a similar analogy for the characteristics of the limit pair for rectangular $\M$. 

\begin{prop}\label{prop:stable_format} \footnote{All proofs are included in the supplemental materials.}
Suppose that $(P,Q)$ is a \textbf{stable} pair of $u\times v$-matrices. 
Then up to permutations, $P$ is a block-wise diagonal matrix 
of the form $P=\diag(B_1, \dots, B_k)$\footnote{The corresponding statement holds for $Q$ too.}, where 
each $B_i$ is row normalized and has a constant column sum denoted by $c_i$.
In particular, $c_i=u_i/v_i$, where $u_i\times v_i$ is the dimension of $B_i$, for $i\in \{1,\dots, k\}$. 
\end{prop}



In addition to providing a convergence format for more general discrete joint distributions, the block diagonal form implies relations between subset of data and concepts that can be leveraged for developing structured models and joint distributions. 

Let $(\LLmat, \TTmat)$ be the limit pair of SK iteration on $\M$. $\LLmat$ and $\TTmat$ must have the same partial pattern of $\M$ as the SK iteration preserves zeros. 
Hence, the existence of a pair of \textit{stable} matrices with partial pattern of $\M$
is necessary for the convergence of SK. 
In Proposition~\ref{EU}, we show that this condition is also sufficient.

\textit{Stable} matrices with partial pattern of $\M$ can be partially ordered with respect to their patterns.
We use $\overbar{\M}$ to denote the matrix obtained from $\M$ by setting elements outside the maximum partial pattern to zeros.
Note that elements outside the maximum partial pattern of a rectangular matrix shall be treated as off-diagonal elements in a square matrix.


\begin{prop} \label{EU}
A non-negative rectangular matrix $\M$ converges to a pair of \textit{stable} matrices 
under SK iteration if and only if there exists a \textit{stable} pair of matrices with partial pattern of $\M$.
\end{prop}
\textit{Proof.} The `only if' direction is clear from the above discussion. 
We now show the `if' direction.
Suppose there exists a \textit{stable} pair $(P,Q)$ such that $P\prec \M$. 
Let $\{\LLmat^1, \TTmat^1, \LLmat^2, \TTmat^2, \dots\}$ be the sequence of matrices generated by SK iteration on $\M$, where $\LLmat^k$ and $\TTmat^k$ are row and column normalized respectively. 
This sequence is bounded since each element of $\LLmat^k$ or $\TTmat^k$ is bounded above by $1$. Hence, according to  Bolzano–Weierstrass theorem, the sequence must have as a limit a pair of matrices (may not be unique). 
Let $(\LLmat, \TTmat)$ and $(\LLmat', \TTmat')$ be two pairs of such limits. To show that they are the same, we only need to 
prove that $\LLmat=\LLmat'$. Lemma~\ref{limit_intermediate} and Remark~\ref{same_pattern} indicate that $\LLmat$ and $\LLmat'$ must have the maximum partial pattern of $\M$, hence, they have the same pattern. Moreover, because they are limits,  $\LLmat$ and $\LLmat'$ are \textit{stable} as well. Therefore, it follows from Proposition~\ref{prop:stable_format} that up to permutations, $\LLmat$ and $\LLmat'$ have the same column sums. 
Further, Lemma~\ref{lemma:diagonal_eq} implies that there exists
$X, Y$ and $X', Y'$ such that $\LLmat=X \overbar{\M} Y$ and $\LLmat'=X' \overbar{\M} Y'$. Therefore, $\LLmat$ and $\LLmat'$ not only have the same row and column sums, but also are diagonally equivalent. Thus, Lemma~\ref{lemma: diag_eq} implies that $\LLmat=\LLmat'$. 
$\Box$

In fact, the existence of a stable pair of matrices with partial pattern of $\M$
is naturally satisfied for all $\M$ under considerations(non-negative matrices without zero rows or zero columns), thus:

\begin{prop}\label{prop:exist_sk_subpattern}
For any matrix $\M$, 
there exists a \textit{stable} pair of matrices $(P,Q)$ such that $P$ and $Q$ have a partial pattern of $\M$.
\end{prop}
Construction of a such $(P,Q)$ is illustrated below. 
\begin{example}\label{eg:subpattern}
Let $\tiny \M=\begin{pmatrix} m_{11}& m_{12} & m_{13}\\ m_{21} & m_{22}&m_{23} \end{pmatrix}$ be a matrix without zero row or zero column. The first two columns are both non-zero implies that up to permutation, either $m_{11}\neq0, m_{12}\neq 0$ or $m_{11}\neq0, m_{22}\neq 0$  . (1)~If $m_{11}\neq0, m_{22}\neq 0$, we may assume that $m_{23}\neq 0$ (up to permutation). In this case, let $\tiny A=\begin{pmatrix} 1& 0& 0\\ 0 & 1&1\end{pmatrix}$.
(2)~Otherwise $m_{11}\neq0, m_{12}\neq 0$. (2-A)~If further $m_{23}\neq 0$, let $\tiny A=\begin{pmatrix} 1& 1& 0\\ 0 & 0&1\end{pmatrix}$.
(2-B)~If $m_{23}= 0$, then $m_{13}\neq 0$. There must exist a non-zero element in the second row of $\M$. Up to permutation,
we may assume that $m_{21}\neq 0$, let $\tiny A=\begin{pmatrix} 0& 1& 1\\ 1 & 0&0\end{pmatrix}$.
In all cases, $A\prec \M$ is block-wise diagonal with each block in the form of a row or column vector. 
Let $P,Q$ be row and column normalization of $A$ respectively. 
It is straightforward to check that $(P,Q)$ is \textit{stable}. 
\end{example}
Propositions~\ref{EU} and \ref{prop:exist_sk_subpattern} together imply our main result:
\begin{thm}\label{main}
Every rectangular matrix converges to a pair of stable matrices under SK iteration.
\end{thm}
\begin{remark}\label{rmk:diff_scalar_SK}
Theorem~\ref{main} is different from the classical convergence result for \textit{scalar Sinkhorn iteration} \citep{menon1969spectrum}.
Let $\M$ be a $u\times v$-matrix, $\mathbf{r}=(r_1, \dots, r_u)^T$ be column vector and $\mathbf{c}=(c_1, \dots, c_v)$ row vector. Similarly to the (regular) SK iteration, scalar SK iteration also alternates between row and column normalizing steps. 
However in each step of scalar SK, row-$i$ (column-$j$) is normalized to have sum $r_i$ (sum $c_j$) instead of $1$. 
The convergence~\footnote{Here convergence means the sequence generated by the iterative process converges to a single matrix.} of scalar SK on a given tuple $(\M, \mathbf{r}, \mathbf{c})$ has been intensively studied.
A complete summary of equivalent convergence criteria are described in \citep{idel2016review}. 
Unfortunately, we can not simply apply these existing results:
(1) Normalizing with respect to $\mathbf{r}$ and $\mathbf{c}$ has no statistical basis for our setting.
(2) The convergence criteria are hard to verify. 
(3) For a given model, the teacher's data selection matrix needs not to be the same as the learner's concepts inferring matrix.
\end{remark}
\begin{cor}\label{cor:M_barM_SKsame}
SK iteration of $\M$ and $\overbar{\M}$ converge to the same limit. Therefore, $CI(\M)=CI(\overbar{\M})$.
\end{cor}

\begin{remark}\label{rmk:finite_step}
Corollary~\ref{cor:M_barM_SKsame} indicates that the elements outside the maximum partial pattern of $\M$ have no effect on the limit, and thus on \textit{Cooperative Index}. For instance, in square matrices, such elements are precisely positive off diagonal entries. They are easy to detect using ideas from graph theory \citep{dulmage1958coverings}. 
Being able to pass to the maximal partial pattern 
makes the cooperative inference much more feasible.  The convergence of $\overbar \M$ is linear, where as the convergence of $\M$ slower \citep{soules1991rate}.
\end{remark}





In the rest of this paper, we assume $\M$ is square. With machinery developed in this section, similar analysis can be made for rectangular matrices. 

\section{Equivalence and sensitivity}\label{sec:cr_sn}
We first introduce cross ratio equivalence and show that models whose joint distribution matrices are cross ratio equivalent, are the same under cooperative inference. Further, we will show that cooperative inference on models is robust to small perturbations on the joint distribution matrix $\M$. 
These features are essential because in most realistic situations we only have access to noisy data points, and because they provide flexibility in model choice by allowing selection of any joint distribution in a cross ratio equivalent class. 



\subsection{Cross Ratio Equivalence}\label{sec:CR_eq}
Intuitively, given a model, SK iteration is a process that selects a representation for two cooperative agents. 
We develop a method to characterize the models that yield to the same representation.

SK iteration can be interpreted as a map between the initial and the limit matrices.
Let $\A$ be the set of $n\times n$ matrices that has at least one positive diagonal, $\overbar{\A}\subset \A$ be the set of $n\times n$ matrices with \textit{total support} (Definition~\ref{def:positive_diag}) and $\B$ be the set of $n\times n$ doubly stochastic matrices. According to \citep{sinkhorn1967concerning}, SK iteration of any $\M\in \A$ converges to a unique matrix $\M^*\in \B$. Hence SK iteration can be viewed as a map $\Phi$ from $\A$ to $\B$ where $\Phi(\M)=\M^*$.


It is important to note that $\Phi$ is not injective. For instance, in Example~\ref{eg:off_diag*} below, with any choices of $m_{12}$ and $m_{32}$, $\M$ maps to the same image under $\Phi$. For a matrix $\LLmat\in \B$, $\Phi^{-1}(\LLmat)$ is used to denote the set of all matrices in $\A$ that map to $\LLmat$.

We will now introduce the notion-\textit{cross ratio equivalence} between square matrices and 
show that the preimage set of a matrix $\LLmat\in \B$ can be completely characterized by its \textit{cross ratios}. 
\begin{definition}\label{def:cross ratios}
Let $A, B$ be two $n\times n$ matrices and $D^A_1=\{a_{1,\sigma(1)},\dots, a_{n, \sigma(n)}\}$ and $D^A_2=\{a_{1,\sigma'(1)},\dots, a_{n, \sigma'(n)}\}$ be two positive diagonals of $A$ determined by permutations $\sigma, \sigma'\in S_n$ (Definition~\ref{def:positive_diag}). 
Denote the products of elements on $D^A_1$ and $D^A_2$ by $d^A_1=\Pi_{i=1}^{n}a_{i, \sigma(i)}, d^A_2=\Pi_{i=1}^{n}a_{i, \sigma'(i)}$ respectively. Then $\CR(D^A_1, D^A_2)=d^A_1/d^A_2$ is called the \textbf{cross ratio} between $D^A_1$ and $D^A_2$ of $A$. Further, let the diagonals in $B$ determined by the same $\sigma$ and $\sigma'$ be 
$D^B_1=\{b_{1,\sigma(1)},\dots, b_{n, \sigma(n)}\}$ and $D^B_2=\{b_{1,\sigma'(1)},\dots, b_{n, \sigma'(n)}\}$. We say $A$ is \textbf{cross ratio equivalent} to $B$, denoted by $A\overset{cr}{\sim} B$, if $d_i^A\neq 0 \Longleftrightarrow d_i^B\neq 0$
and $\CR(D^A_1, D^A_2)=\CR(D^B_1, D^B_2)$ holds for 
any $D^A_1$ and $D^A_2$.
\end{definition}

\begin{example}
Let $\tiny A=\begin{pmatrix} 3 & 2 & 1\\ 0 & 1&1\\1&0&1 \end{pmatrix},
B=\begin{pmatrix}9 & 20 & 6\\ 0 & 5&3\\2&0&4 \end{pmatrix}$.
$A$ has three positive diagonals $D^A_1=\{a_{11}, a_{22}, a_{33}\}$, $D^A_2=\{a_{12}, a_{23}, a_{31}\}$ and
$D^A_3=\{a_{13}, a_{22}, a_{31}\}$ with $d^A_1=3, d^A_2=2,d^A_3=1$. $B$ has three corresponding positive 
diagonals $D^B_1, D^B_2$ and $D^B_3$ with $d^B_1=180, d^B_2=120,d^B_3=60$. 
It is easy to check that $\CR(D^A_i,D^A_j)=\CR(D^B_i,D^B_j)$ for any $i,j\in \{1,2,3\}$.
Hence $A$ is \textit{cross ratio equivalent} to $B$.
\end{example}

\begin{remark}
(1) Definition~\ref{def:cross ratios} implies that if $A\crs B$, then $\overbar{A}$ and $\overbar{B}$ (Definition~\ref{def:positive_diag})
must have the same pattern. Otherwise there exists a positive diagonal $D_1^A$ of $A$ (or B) whose corresponding diagonal $D_1^B$ in B (or A) contains zero ($d^A_1\neq 0$ whereas $d^B_1=0$). 
\newline
(2)  
Let $A$ and $B$ be matrices with the same pattern.
Assume they both have $N_d$ positive diagonals. 
To determine whether $A$ is \textit{cross ratio equivalent} to $B$, instead of examining $\binom{N_d}{2}$ pairs of cross ratios, it is sufficient to check whether $CR(D^A_1, D^A_i)=CR(D^B_1, D^B_i)$, $i\in \{1, \dots, N_d\}$ holds for a fixed positive diagonal $D^A_1$. 
\end{remark}

\begin{prop}\label{prop:preimage}
Let $\M\in A$ be a consistency matrix and $\LLmat\in \B$ be a doubly stochastic matrix.
Then $\M \in \Phi^{-1}(\LLmat)$ if and only if $\M$ is \textit{cross ratio equivalent} to $\LLmat$.
\end{prop} 
\textit{Sketch of proof.}
Let $\M \in \Phi^{-1}(\LLmat)$, we now show they have the same cross ratios.
Since $\M$ and $\overbar{\M}$ have exactly the same positive diagonals, we may assume that $\M$ has total support.
Hence, \citep{sinkhorn1967concerning} implies that 
there exist diagonal matrices $X=\diag(x_1, \dots, x_n)$ and $Y=\diag(y_1, \dots, y_n)$ such that $\M=X \LLmat Y$.
In particular, $m_{ij}=x_i\times l_{ij} \times y_j$ holds, for any element $m_{ij}$. 
Let $D^{\M}_1=\{m_{i,\sigma(i)}\}$, $D^{\M}_2=\{m_{i,\sigma'(i)}\}$
be two positive diagonals of $\M$ and $D^{\LLmat}_1=\{l_{i,\sigma(i)}\}$, $D^{\LLmat}_2=\{l_{i,\sigma'(i)}\}$ be the corresponding positive diagonals in $\LLmat$. Then:
\begin{equation*}\label{eq:same_cr}
\footnotesize
 \begin{split}
CR(D^{\M}_1,D^{\M}_2)& =\frac{\Pi_{i=1}^{n}m_{i,\sigma(i)}}{\Pi_{i=1}^{n}m_{i,\sigma'(i)}}
=\frac{\Pi_{i=1}^{n}x_i\times l_{i,\sigma(i)} \times y_{\sigma(i)}}{\Pi_{i=1}^{n} x_i\times l_{i,\sigma'(i)}\times y_{\sigma'(i)}}\\
& = \frac{\Pi_{i=1}^{n}x_i\times \Pi_{i=1}^{n} l_{i,\sigma(i)} \times \Pi_{i=1}^{n} y_{\sigma(i)}}{\Pi_{i=1}^{n} x_i\times \Pi_{i=1}^{n} l_{i,\sigma'(i)}\times \Pi_{i=1}^{n} y_{\sigma'(i)}}\\
&=\frac{\Pi_{i=1}^{n}l_{i,\sigma(i)}}{\Pi_{i=1}^{n}l_{i,\sigma'(i)}}=CR(D^{\LLmat}_1,D^{\LLmat}_2) \hspace{0.5in} \Box
\end{split}
\end{equation*}

\begin{cor}\label{cor:total_support}
For $\M_1, \M_2 \in \A$, if $\M_1\crs \M_2$ then $\CI(\M_1)=\CI(\M_2)$.
\end{cor}
Proposition~\ref{prop:preimage} captures the key ingredient, \textit{cross ratios}, of a model.
It indicates that cross ratio equivalent models can be treated the same for cooperative agents.
Corollary~\ref{cor:total_support} implies that their cooperative indices are the same and hence they have the same communication effectiveness. This can be very useful in practice: (1) Models with same representation can be effectively categorized, which avoids unnecessary implementation of similar models;
(2) Models can be freely modified as long as the cross ratios are preserved which may increase computational efficiency. 

%



\subsection{Sensitivity Analysis}\label{sec:sensitivity analysis}
We now investigate sensitivity of $\Phi$ to perturbation of $\M$. 
Without loss of generality, we will assume that only one element in $\M$ is perturbed at a time as other perturbations may be treated as compositions of such.
Let $\M^{\e}=(m^{\e}_{ij})$ be a matrix obtained by varying the element $m_{st}$ of $\M=(m_{ij})$ by $\e$, i.e. 
$m^{\e}_{st}=m_{st}+\e $ and $m_{ij}=m^{\e}_{ij}$ for $(i,j)\neq (s,t)$.
We may also assume that $\e>0$. Otherwise we may view $\M$ as a matrix obtained from a positive perturbation on $\M^{\e}$.

Proposition~\ref{prop:preimage} indicates that $\Phi$ is robust to any amount of perturbation on \textit{off diagonal} elements.
In more detail, suppose that both $m^{\e}_{st}$ and $m_{st}$ are off diagonal elements of $\M^{\e}$ and $\M$ respectively. 
Then $\overbar{\M}=\overbar{\M}^{\e}$ $\implies$ $\Phi(\M)=\Phi(\overbar{\M})=\Phi(\overbar{\M}^{\e})=\Phi(\M^{\e})$ $\implies$ $\CI(\M)=\CI(\M^{\e})$. Thus we have: 

\begin{prop}\label{prop:robust_off_diag}
Cooperative Inference is robust to any amount of off diagonal perturbations on $\M$.
\end{prop}

\begin{example}\label{eg:off_diag*}
Let $\tiny{\M=
\begin{blockarray}{cccc}
&h_1 & h_2 & h_3  \\
\begin{block}{c (ccc)}
 d_1 & 1 & * &1\\
 d_2 & 0 &1 & 0\\
 d_3 & 1 & * &1\\ \end{block} \end{blockarray}}$ be a consistency matrix. 
Suppose that the consistency between $d_1, d_3$ and $h_2$ can not be properly measured.
With Proposition~\ref{prop:robust_off_diag}, $\CI(\M)$ can still be easily obtained:
$\tiny \overbar{\M}=\begin{pmatrix} 1 & 0 & 1\\ 0 & 1&0\\1&0&1 \end{pmatrix}$ converges to 
$\tiny \overbar{\M}^*=\begin{pmatrix} 0.5& 0 & 0.5\\ 0 & 1&0\\0.5 &0&0.5 \end{pmatrix}$ in one step of SK iteration.
So we have that $\CI(\M)=\CI(\overbar{\M})=(4\times 0.5^2+1^2)/3=2/3$.
\end{example}

Proposition~\ref{prop:robust_off_diag} is not only important for sensitivity analysis, but also practical to efficiently perform cooperative inference as mentioned in Remark~\ref{rmk:finite_step}.
For instance, if one $*$ in Example~\ref{eg:off_diag*} is positive, it takes infinite many steps of SK iteration for $\M$ to reach its limit, whereas it takes only one step for~$\overbar{\M}$. 

Proposition~\ref{prop:robust_off_diag} also implies the main theorem in \citep{YangYGWVS18} stating $\CI(\M)$ is optimal if $\M$ is a permutation of a triangular matrix. For an  $n\times n$ \textit{triangular} matrix $\M=(m_{ij})$, all the elements except $m_{i,i}$ are off diagonal.
To efficiently apply cooperative inference, one only needs to consider $\overbar{\M}=\diag(m_{11},\dots, m_{nn})$. SK iteration on $\overbar{\M}$ converges to $I_n=\diag(1,\dots, 1)$ in one step. Therefore, we have $\CI(\M)=1$.

By analogy, Corollary~\ref{cor:M_barM_SKsame} implies that CI is robust to any perturbation on elements that are off maximal partial pattern for \textit{rectangular} matrices as well.


Perturbations for \textit{on-diagonal} elements are more complicated and interesting. 
To obtain $\M^{\e}$, one may either perturb an on-diagonal element of $\M$ or perturb a zero element of $\M$ introducing a new diagonal(s) for $\M^{\e}$ .

A celebrated result in \citep{sinkhorn1972continuous} shows that $\Phi:\A\to \B$ is a continuous function:
\begin{thm}[Continuity of SK iteration]\label{thm:continous_sk}
$\Phi(\M^{\e})$ converges to $\Phi(\M)$ as $\M^{\e}\to \M$.
\end{thm}

Here, distance between matrices are measured by the maximum element-wise difference, e.g. $d(\M, \M^{\e})=\e$. 

This implies that small on-diagonal perturbations on a model with joint distribution $\M$, yield close solutions for cooperative inference. 
\begin{example}\label{eg:on_diag_epsilon}
Let $\tiny \M=\begin{pmatrix} 1 & 1 &0\\0& 1&1\\1&0&1 \end{pmatrix}$, $\tiny\M^{\e_1}=\begin{pmatrix} 1.5 & 1&0 \\0& 1&1\\1&0&1 \end{pmatrix}$, $\tiny\M^{\e_2}=\begin{pmatrix} 1.1 & 1&0 \\0& 1&1\\1&0&1 \end{pmatrix}$, $\tiny\M^{\e_3}=\begin{pmatrix} 1 & 1&0.1 \\0& 1&1\\1&0&1 \end{pmatrix}$ and 
$\tiny\M^{\e_4}=\begin{pmatrix} 1 & 1&0.5 \\0& 1&1\\1&0&1 \end{pmatrix}$. Apply SK iterations on $\M$ and $ \M^{\e_i}$,  we have:
$\tiny\Phi(\M)=\begin{pmatrix} 0.5 & 0.5 &0\\0& 0.5&0.5\\0.5&0&0.5 \end{pmatrix}$,
$\tiny\Phi(\M^{\e_1})=\begin{pmatrix} 0.534 & 0.466&0 \\0& 0.534&0.466\\0.466&0&0.534 \end{pmatrix}$,  
$\tiny\Phi(\M^{\e_2})=\begin{pmatrix} 0.508 & 0.492&0 \\0& 0.508&0.492\\0.492&0&0.508 \end{pmatrix}$, 
$\tiny\Phi(\M^{\e_3})=\begin{pmatrix}  0.478& 0.478& 0.044 \\0& 0.522 8&0.478\\0.522,&0&0.478 \end{pmatrix}$, 
$\tiny\Phi(\M^{\e_4})=\begin{pmatrix}  0.423& 0.423&  0.155 \\0 & 0.577& 0.423\\0.577& 0& 0.423 \end{pmatrix}$. 
It is clear that for perturbations on the same location, the variation on the limit matrix decreases as the size of the perturbation gets smaller. Moreover, perturbations of the same size cause different variations on the limits depending on whether a new diagonal is introduced. For instance, $\M^{\e_4}$ introduces a new diagonal to $\M$ whereas $\M^{\e_2}$ does not.
Although both are $0.5$ away from $\M$, 
after SK iteration
$d(\Phi(\M),\Phi(\M^{\e_2}))=0.034$ and $d(\Phi(\M),\Phi(\M^{\e_4}))=0.155$.
\end{example}

In the following example, we illustrate how one may effectively bound the variation in the limit in terms of~$\e$, even for perturbations that introduce new diagonals. However, in general, $\Phi$ is not Lipschitz \footnote{The authors would like to thank Yue Yu for pointing out counterexamples.}. 


\begin{example}\label{eg:new_diag_epsilon}
Let $\tiny \M=\begin{pmatrix} a_{11} & a_{12} &0\\0& a_{22}&a_{23}\\a_{31}&0&a_{33} \end{pmatrix}$ and  $\tiny\M^{\e}=\begin{pmatrix} a_{11} & a_{12} &\e \\0& a_{22}&a_{23}\\a_{31}&0&a_{33} \end{pmatrix}$.
While $\M$ has only two diagonals $D_1$ and $D_2$ with products of elements $d_1=a_{11}a_{22}a_{33}$ and $d_2=a_{12}a_{23}a_{31}$, 
the perturbation introduces one more diagonal $D_3$ with $d_3=\e\cdot a_{22}a_{31}$ to $\M^{\e}$. 
The Birkhoff-von Neumann theorem (Theorem~\ref{thm:BN}) guarantees that   
doubly stochastic matrices $\Phi(\M)$ and $\Phi(\M^{\e})$ can be written as \textit{convex combinations} of permutation matrices as shown below:

 $\tiny \Phi(\M)=\theta_1\begin{pmatrix}1 & 0 &0\\0& 1&0\\0&0&1 \end{pmatrix}+\theta_2\begin{pmatrix}0 & 1 &0\\0& 0&1\\1&0&0 \end{pmatrix}=\begin{pmatrix} \theta_1 & \theta_2 &0\\0& \theta_1&\theta_2\\\theta_2&0&\theta_1 \end{pmatrix} $
\begin{equation*}
    \begin{split}
     \Phi(\M^{\e})&=\alpha_1 \tiny{\begin{pmatrix} 1 & 0 &0\\0& 1&0\\0&0&1 \end{pmatrix}}+\alpha_2\begin{pmatrix}0 & 1 &0\\0& 0&1\\1&0&0 \end{pmatrix}+\alpha_3\begin{pmatrix} 0& 0 &1\\0& 1&0\\1&0&0 \end{pmatrix} \\
      &= \scriptsize{\begin{pmatrix}\alpha_1 & \alpha_2 &\alpha_3\\0& \alpha_1+\alpha_3&\alpha_2\\\alpha_2+\alpha_3&0&\alpha_1 \end{pmatrix}},\\
    \end{split}
\end{equation*}
where $\theta_1+\theta_2=1, \alpha_1+\alpha_2+\alpha_3=1$ and $\theta_i, \alpha_j>0$. 
Notice that the variation between $\Phi(\M)$ and $\Phi(\M^{\e})$ is caused by $\alpha_3$, 
we will now derive an upper bound for $\alpha_3$. 
Since $\Phi$ preserves cross ratios, evaluating a cross ratio, for example $\CR(D_3, D_1)$, in both $\Phi(\M^{\e})$ and $\M^{\e}$ 
we have that:
\begin{equation}\label{eq:cr_Me}
\footnotesize{\frac{\alpha_3(\alpha_2+\alpha_3)}{\alpha^2_1}=\frac{d_3}{d_1}=\e\cdot \frac{a_{22}a_{31}}{a_{11}a_{22}a_{33}}:=\e\cdot A_1}
\end{equation}
Since $\alpha_1+\alpha_2+\alpha_3=1$, we may assume that $\small{\alpha_1<1/2}$ 
Substituting $\small {\alpha_2+\alpha_3=1-\alpha_1}$ into Equation~\eqref{eq:cr_Me}, we get $\frac{\alpha_3(1-\alpha_1)}{\alpha^{2}_1}=\e\cdot A_1$ and this implies that:
\begin{equation*}\label{eq:alpha_3}
  \footnotesize  \alpha_3=\e\cdot A_1\cdot \frac{\alpha^{2}_1}{1-\alpha_1}\leq \e\cdot A_1 \cdot \frac{1}{2}\,, 
\end{equation*}
where the last `$\leq$' holds because $\frac{\alpha^{2}_1}{1-\alpha_1}$ reaches its maximum at $\alpha_1=\frac{1}{2}$ for $\alpha_1<\frac{1}{2}$. Thus, $\alpha_3$ is bounded above by a constant multiple of $\e$.



\end{example}


The next proposition explores how sensitive $\Phi$ is to perturbations on its images. Thus, given two doubly stochastic matrices in $\B$, we will measure the distance between their preimages under $\Phi$.  


\begin{prop}\label{prop:dist_preimage}
Let $\LLmat^1, \LLmat^2\in \B$. If $d(\LLmat^1, \LLmat^2) \leq \epsilon$, for any $\M^1\in \Pn^{-1}(\LLmat^1)$ with total support, there exist a $\M^2\in \Pn^{-1}(\LLmat^2)$ and a constant $C$ such that $d(\M^1, \M^2)\leq C\cdot \e$.
\end{prop}

In fact, restricting to matrices with total support, $\Phi$ can be amended into a homeomorphism (see Supplemental Materials).
Viewing SK iteration as a representation selecting process,  
the homeomorphic property of $\Phi$ indicates that such process preserves important information needed to reconstruct the original model.


\section{Lower Bound for CI}\label{sec:bounds}

\textit{Cooperative Index} measures the effectiveness of the cooperative communication. However, for a given consistency matrix $\M$, in order to calculate $\CI(\M)$ one needs to obtain $\Phi(\M)$ by \textit{SK iterations}, which sometimes can be an expensive process. 
We provide bounds on $\CI(\M)$ that do not require computing SK. 

.


First, we derive a uniform bound for $\CI(\M)$ which only depends on the size of $\M$. 

\begin{prop}
For an $n \times n$ matrix $\M$, $\displaystyle \CI(\M)\geq\frac{1}{n}$ with the equality when $\M$ is uniformly distributed.
\end{prop}
\textit{Proof.}
Let $\M^*=\left(m^*_{ij}\right)_{n\times n}$ be the limit of $\M$ under SK iteration. 
By Generalized Mean Inequality, we have $\left(\frac{\sum_{ij} m^*_{ij}}{n^2}\right)^2\leq \frac{\sum_{ij} (m^*_{ij})^2}{n^2}$. Since $\M^*$ is doubly stochastic, we have $\sum_{ij} m^*_{ij}=n$ and it follows that $\sum (m^*_{ij})^2 \geq 1$. Therefore, $\CI(\M)=\frac{\sum (m^*_{ij})^2}{n} \geq \frac{1}{n}$. $\Box$


Notice that, as the size of $\M$ increases the above bound is not effective. However, the number of positive diagonals a matrix consists can be small regardless of its size. Next, we provide another lower bound for $\CI(\M)$ that depends only on the number of positive diagonals. 

\begin{prop}\label{no.of_diag}
For an $n\times n$ matrix $\M$ with $d$ positive diagonals, $\displaystyle \CI(\M)\geq 1/d$.
\end{prop}
\textit{Proof.}
Since $\M$ is a square matrix, the limit of SK iteration is a unique doubly stochastic matrix $\M^*$. 
Therefore, by Birkhoff-von Neumann theorem
$\M^*=\sum_{i=1}^{d} \theta_iP_i$ and by Definition~\ref{def:CIndex} we have: 
\begin{align*}
    \mbox{CI}(\M) &=\frac{1}{n} \M^*\odot \M^* = \frac{1}{n} \left(\sum_{i=1}^{d} \theta_iP_i\right)\odot \left(\sum_{i=1}^{d} \theta_iP_i\right)\\
    & \geq \frac{1}{n} \sum_i^d \theta_iP_i \odot \theta_iP_i
    \overset{(1)}{=}\sum_i^d  \theta_i^2 
    \overset{(2)}{\geq} \frac{1}{d}
\end{align*}
Equality~(1) holds because each $P_i$ is a permutation matrix and so $P_i\odot P_i=n$.
Inequality~(2) is obtained from Generalized Mean Inequality as $\sum_{i=1}^{d}\theta_{i}=1$. $\Box$

Such a bound makes sense because $\CI$ measures the effectiveness of the cooperative communication.
Each diagonal is a representation for communication. $\CI(\M)$ decreases as the number of diagonals increases. 
The optimal $\CI(\M)$ is achieved when $\M$ has only one diagonal, i.e.$\M$ is upper triangular up to permutation.
\begin{example}
Consider 
$\M$ in Example~\ref{eg:off_diag*}. We have $n=3, d=2$ and  
$\CI(\M)=\left(0.5^2\times 4+1\right)/3=2/3>1/2=1/d >1/3=1/n$.
\end{example}
Above example shows that when the number of diagonals is small, Proposition~\ref{no.of_diag} provides a good bound. However, counting the number of diagonals of an $n\times n$ matrix can also be computationally expensive. 
Next, we provide a much more accessible bound. 
\begin{definition}
An $n \times n$ matrix $A$ is \textbf{indecomposable} if there exists no permutation matrices $P$ and $Q$ such that $\tiny PAQ= \begin{pmatrix}
A_{11} & 0\\
A_{21} & A_{22}
\end{pmatrix}$ where, $A_{11}$ and $A_{22}$ are square submatrices.
\end{definition}

 


\begin{prop}\label{pos_ele}
For any $n \times n$ matrix $\M$, $\CI(\M)\geq \frac{1}{\eta-2n+\tau+1}$, where $\eta$ is the number of positive elements and $\tau$ is the number of indecomposable components.
\end{prop}

\textit{Proof.} Let $\M^*$ be the SK limit of $\M$ and $\eta^*$ and $\tau^*$ be the number of positive elements and the number of indecomposable components in $\M^*$, respectively. 
Then according to \citep{brualdi1982notes}, 
$\M^*$ has a Birkhoff-von Neumann decomposition with $k$ permutation matrices, where
$k \leq \eta^*+\tau^*-2n+1$. Further note that $\eta + \tau \leq \eta^*+\tau^*$.
Hence, similarly as in the proof of Proposition~\ref{no.of_diag}, 
we have that $\CI(\M)=\CI(\M^*)\geq \frac{1}{k}\geq \frac{1}{\eta^*+\tau^*-2n+1} \geq \frac{1}{\eta+\tau-2n+1}$. $\Box$

\begin{example}Consider an $n\times n$ matrix $\M$ of the form 
\begin{minipage}{0.3\linewidth}   
   \begin{tikzpicture}\tiny [baseline=(current bounding box.center)]
\matrix (m) [matrix of math nodes,nodes in empty cells,right delimiter={)},left delimiter={(} ]{
* & & &  &&  & * \\
  & & && & &  \\
  && *& *& *& &   \\
  & & *& *& *&&   \\
&& *& *&  *&&   \\
 & & && & &  \\
 * & & &  &&  & * \\
} ;
 \draw[loosely dotted] (m-1-1)-- (m-3-3);
 \draw[loosely dotted] (m-1-7)-- (m-3-5);
 \draw[loosely dotted] (m-7-1)-- (m-5-3);
 \draw[loosely dotted] (m-7-7)-- (m-5-5);
\end{tikzpicture}
  \end{minipage}
  \begin{minipage}{0.6\linewidth}
where, any * is a positive number and the rest are zeros. Notice that, it quickly becomes challenging to count $d$ when $n$ is large.
  \end{minipage}
When $n=5$, we have $\eta=13, \tau=2, d=12$, and so $\CI(\M)\geq 1/(\sigma +\tau -2n+1)=1/6>1/d$.     
\end{example}
\section{Connections to other work}\label{sec:connections}
\textbf{Geometric interpretation.} Cooperative inference is intuitive given the geometric interpretation of SK iteration, which has
been long known and favored in the study of contingency tables \citep{fienberg1970iterative, borobia1998matrix}. 
Each joint distribution matrix $\M=(m_{ij})$ of dimension $u\times v$ can be viewed as a point in the $(uv-1)$-dimension simplex,
$\mathcal{S}_{uv}=\small{\{(m_{11},\dots, m_{uv}):m_{ij}\geq 0, \sum_{ij} m_{ij}=1\}}$.
In \citep{fienberg1968geometry}, the author showed that, in $\mathcal{S}_{uv}$, positive\footnote{Every element is positive.} matrices 
with the same cross-product ratios\footnote{Two matrices with the same cross-product ratios must be cross ratio equivalent (Definition~\ref{def:cross ratios}).} form a special case of determinantal manifold $\mathcal{H}$, which is studied in \citep{room1938geometry}.
In particular, for the $2\times 2$ case, authors of \citep{fienberg1970iterative} built a homeomorphic map from $\mathcal{H}$ to the unit square and illustrated the convergence path of successive SK iterations in the unit square. 
Similarly, non-negative joint-distribution matrices with the same pattern locate on a lower-dimension face of $\mathcal{S}_{uv}$ and 
matrices with the same cross ratios form a further subspace.


\textbf{Optimal transport.} Choosing a suitable distance to compare probabilities is a key problem in statistical machine learning. When the probability space is a metric space, optimal transport distances (earth mover’s in computer vision) define a powerful geometry to compare probabilities \citep{villani2008optimal}. Optimal transport distances are a fundamental family of distances for probability
measures and histograms of features. \citep{cuturi2013sinkhorn} proposed a new family
of optimal transport distances, \textbf{Sinkhorn distance}, that look at transport problems from a maximum entropy
perspective. 
The resulting optimum is a proper distance
which can be computed through Sinkhorn’s matrix scaling. 
Let $C$ be the cost matrix, $\mathbf{r}$ and $\mathbf{c}$ be the marginal distributions for a given optimal transport problem. The matrix $\M^*$ that optimizes the Sinkhorn distance can be obtained by applying $(\mathbf{r},\mathbf{c})$-scalar SK iteration on $\M=e^{-\lambda \cdot C}$, where $\lambda$ is the regularization parameter. 
\cite{cuturi2013sinkhorn} proved that this Sinkhorn algorithm can be computed at a speed that is several orders of magnitude faster than that of transport solvers.

Optimal transport with Sinkhorn distance provides a powerful tool for domain adaptation. 
\cite{courty2015optimal} proposed the following method: first link two domains based on prior knowledge (build an initial cost matrix $C$); then learn an optimal distribution matrix $\M^*$ (w.r.t, Sinkhorn distance) from one domain to the other by applying scalar SK iteration on 
$\M=e^{-\lambda \cdot C}$. If certain transports should never happen, i.e. elements of $C$ are allowed to be $\infty$, then 
the corresponding $\M$ will be a sparse matrix. Remark~\ref{rmk:diff_scalar_SK} notes that scalar SK iteration of a sparse $\M$ may not converge and the convergence criteria can not be easily verified. 
Whereas, in the case that both domains have uniform marginal distributions, 
Theorem~\ref{main} guarantees the existence of the optimal distribution matrix $\M^*$ for any choice of the cost matrix. 
In the sparse case, the convergence rate can be further sped up by first identifying and removing off-diagonal elements, i.e. turning $\M$ into $\overbar{\M}$, then applying SK iteration on $\overbar{\M}$ as in Remark~\ref{rmk:finite_step}. 
More importantly, our results in Section~\ref{sec:cr_sn} capture the essential features of the Sinkhorn distance approach.
Proposition~\ref{prop:preimage} implies that cost matrices that are cross ratio equivalent lead to the same optimal transport. Proposition~\ref{thm:continous_sk} indicates that optimal distribution matrix $\M^*$ is continuous to the choice of regularization parameter $\lambda$, which can be used to discretize the range of~$\lambda$.

\textbf{Importance Sampling} Cooperative inference can be interpreted as selection of optimal distributions for importance sampling. 
A straightforward view is to consider Equation~\eqref{eq:LT}.
Given a joint distribution $\M$, let $\TTmat^1$ be the column normalization of $\M$. 
The $ij^{\text{th}}$-element $P_{\TTmat^1}(d_i|h_j)$ of $\TTmat^1$ can be viewed as the teacher's initial probability of selecting $d_i$ to convey $h_j$.  Once $d_i$ is observed, the learner needs to sample a concept to match $d_i$.
Assume that the learner's prior on $\conceptSpace$ is uniform. 
To minimize the variance, the learner should sample from the optimal distribution:
$P_{\LLmat^1}(h_j|d_i)=\frac{P_{\TTmat^1}(d_i|h_j)}{\sum_{j}P_{\TTmat^1}(d_i|h_j)}$.
Thus the optimal learner's matrix $\LLmat^1$ is the row normalization of $\TTmat^1$.
Similarly, based on $\LLmat^1$, to reduce variance, the teacher should sample according to the matrix $\TTmat^2$, the column normalization of $\LLmat^1$. 
This alternating process is precisely SK iteration. 
So, the solution of cooperative inference is not only the stable limit of a sequence of optimal distributions for individual $d$ and $h$, but also the only doubly stochastic matrix
cross ratio equivalent to $\M$.

A more subtle and interesting version of importance sampling is also achieved by cooperative inference.
Let $\M$ be an $n\times n$ joint distribution.
Suppose that the teacher aims to convey the whole set of $n$ concepts simultaneously.
To do so, the teacher must teach $n$ different data points at once---one data point per concept. 
This is equivalent to picking a map from $\dataSpace$ to $\conceptSpace$, i.e. a permutation $\sigma \in S_n$
as $|\dataSpace|=|\conceptSpace|=n$.
Then $P_{T}(\dataSpace_{\sigma}|\conceptSpace_T)=\Pi_{i} P_{T}(d_{\sigma(i)}|h_i)$ is the probability that the teacher picks $\sigma$ 
to teach and $P_{L}(\conceptSpace_L|\dataSpace_{\sigma})=\Pi_{i} P_{L}(h_i|d_{\sigma(i)})$ is the probability that given $\sigma$, the learner's inference completely matches the teacher's intention. 
Therefore, in order to efficiently estimate the communication accuracy, $P(\conceptSpace_L|\conceptSpace_T)=\sum_{\sigma\in S_n} P_{L}(\conceptSpace_L|\dataSpace_{\sigma})P_{T}(\dataSpace_{\sigma}|\conceptSpace_T)$,
one must sample permutations that make large positive contributions to the summation.
Such an importance sampling is attained by Cooperative inference for the following reasons.
(1) SK iteration completely removes the probability of sampling off-diagonal elements. Thus a $\sigma$
will be sampled only if it could lead to a prefect teaching.
(2) \cite{beichl1999approximating} proved that the limit of SK iteration 
maximizes entropy for doubly stochastic matrices that have the same pattern as $\M$, and
further they showed that this is the ideal property for sampling positive diagonals.



\textbf{Other connections.} See supplemental materials for pointers to other connections. 
\section{Conclusion}\label{sec:conclusion}
Cooperative inference holds promise as a theory of human-human, human-machine, and machine-machine information sharing. An impediment to realizing this promise is the lack of foundational results related to convergence, robustness, and effectiveness. We have addressed each of these limitations, including specific results showing the convergence of Cooperative Inference via SK iteration for any rectangular matrix, equivalence classes of models in terms of their cross-ratios, continuity of SK iteration which implies stability to perturbation, and several different bounds on the effectiveness of Cooperative Inference that can be derived from the original model. We also demonstrated connections and implications through geometric interpretations of Cooperative inference, optimal transport, and importance sampling. Important open questions include developing methods for modifying machine learning models to increase the efficacy and furthering our understanding of the representational implications of Cooperative Inference. 

\noindent
\textbf{Acknowledgements}

This material is based on research sponsored by the Air Force Research Laboratory and DARPA under agreement number FA8750-17-2-0146. The U.S. Government is authorized to reproduce and distribute reprints for Governmental purposes notwithstanding any copyright notation thereon. This research was also supported by NSF SMA-1640816, and NSF NSF-1549981.
\newpage
\bibliographystyle{plainnat} 
\bibliography{references}

\begin{thebibliography}{49}
\providecommand{\natexlab}[1]{#1}
\providecommand{\url}[1]{\texttt{#1}}
\expandafter\ifx\csname urlstyle\endcsname\relax
  \providecommand{\doi}[1]{doi: #1}\else
  \providecommand{\doi}{doi: \begingroup \urlstyle{rm}\Url}\fi

\bibitem[Beichl and Sullivan(1999)]{beichl1999approximating}
Isabel Beichl and Francis Sullivan.
\newblock Approximating the permanent via importance sampling with application
  to the dimer covering problem.
\newblock \emph{Journal of computational Physics}, 149\penalty0 (1):\penalty0
  128--147, 1999.

\bibitem[Berry et~al.(2007)Berry, Browne, Langville, Pauca, and
  Plemmons]{berry2007algorithms}
Michael~W Berry, Murray Browne, Amy~N Langville, V~Paul Pauca, and Robert~J
  Plemmons.
\newblock Algorithms and applications for approximate nonnegative matrix
  factorization.
\newblock \emph{Computational statistics \& data analysis}, 52\penalty0
  (1):\penalty0 155--173, 2007.

\bibitem[Bonawitz et~al.(2011)Bonawitz, Shafto, Gweon, Goodman, Spelke, and
  Schulz]{bonawitz2011double}
Elizabeth Bonawitz, Patrick Shafto, Hyowon Gweon, Noah~D Goodman, Elizabeth
  Spelke, and Laura Schulz.
\newblock The double-edged sword of pedagogy: Instruction limits spontaneous
  exploration and discovery.
\newblock \emph{Cognition}, 120\penalty0 (3):\penalty0 322--330, 2011.

\bibitem[Borobia and Cant{\'o}(1998)]{borobia1998matrix}
Alberto Borobia and Rafael Cant{\'o}.
\newblock Matrix scaling: A geometric proof of sinkhorn's theorem.
\newblock \emph{Linear algebra and its applications}, 268:\penalty0 1--8, 1998.

\bibitem[Brualdi(1982)]{brualdi1982notes}
Richard~A Brualdi.
\newblock Notes on the birkhoff algorithm for doubly stochastic matrices.
\newblock \emph{Canadian Mathematical Bulletin}, 25\penalty0 (2):\penalty0
  191--199, 1982.

\bibitem[Buchsbaum et~al.(2011)Buchsbaum, Gopnik, Griffiths, and
  Shafto]{buchsbaum2011children}
Daphna Buchsbaum, Alison Gopnik, Thomas~L Griffiths, and Patrick Shafto.
\newblock Children’s imitation of causal action sequences is influenced by
  statistical and pedagogical evidence.
\newblock \emph{Cognition}, 120\penalty0 (3):\penalty0 331--340, 2011.

\bibitem[Courty et~al.(2015)Courty, Flamary, Tuia, and
  Rakotomamonjy]{courty2015optimal}
Nicolas Courty, R{\'e}mi Flamary, Devis Tuia, and Alain Rakotomamonjy.
\newblock Optimal transport for domain adaptation.
\newblock \emph{arXiv preprint arXiv:1507.00504}, 2015.

\bibitem[Csibra and Gergely(2009)]{csibra2009natural}
Gergely Csibra and Gy{\"o}rgy Gergely.
\newblock Natural pedagogy.
\newblock \emph{Trends in cognitive sciences}, 13\penalty0 (4):\penalty0
  148--153, 2009.

\bibitem[Cuturi(2013)]{cuturi2013sinkhorn}
Marco Cuturi.
\newblock Sinkhorn distances: Lightspeed computation of optimal transport.
\newblock In \emph{Advances in neural information processing systems}, pages
  2292--2300, 2013.

\bibitem[Dufoss{\'e} and U{\c{c}}ar(2016)]{dufosse2016notes}
Fanny Dufoss{\'e} and Bora U{\c{c}}ar.
\newblock Notes on birkhoff--von neumann decomposition of doubly stochastic
  matrices.
\newblock \emph{Linear Algebra and its Applications}, 497:\penalty0 108--115,
  2016.

\bibitem[Dulmage and Mendelsohn(1958)]{dulmage1958coverings}
Andrew~L Dulmage and Nathan~S Mendelsohn.
\newblock Coverings of bipartite graphs.
\newblock \emph{Canadian Journal of Mathematics}, 10\penalty0 (4):\penalty0
  516--534, 1958.

\bibitem[Eaves~Jr et~al.(2016)Eaves~Jr, Feldman, Griffiths, and
  Shafto]{Eaves2016c}
Baxter~S Eaves~Jr, Naomi~H Feldman, Thomas~L Griffiths, and Patrick Shafto.
\newblock Infant-directed speech is consistent with teaching.
\newblock \emph{Psychological review}, 123\penalty0 (6):\penalty0 758, 2016.

\bibitem[Fienberg(1968)]{fienberg1968geometry}
Stephen~E Fienberg.
\newblock The geometry of an r$\times$ c contingency table.
\newblock \emph{The Annals of Mathematical Statistics}, 39\penalty0
  (4):\penalty0 1186--1190, 1968.

\bibitem[Fienberg et~al.(1970)]{fienberg1970iterative}
Stephen~E Fienberg et~al.
\newblock An iterative procedure for estimation in contingency tables.
\newblock \emph{The Annals of Mathematical Statistics}, 41\penalty0
  (3):\penalty0 907--917, 1970.

\bibitem[Fisac et~al.(2017)Fisac, Gates, Hamrick, Liu, Hadfield-Menell,
  Palaniappan, Malik, Sastry, Griffiths, and Dragan]{fisac2017pragmatic}
Jaime~F Fisac, Monica~A Gates, Jessica~B Hamrick, Chang Liu, Dylan
  Hadfield-Menell, Malayandi Palaniappan, Dhruv Malik, S~Shankar Sastry,
  Thomas~L Griffiths, and Anca~D Dragan.
\newblock Pragmatic-pedagogic value alignment.
\newblock \emph{arXiv preprint arXiv:1707.06354}, 2017.

\bibitem[Frank and Goodman(2012)]{frank2012predicting}
Michael~C Frank and Noah~D Goodman.
\newblock Predicting pragmatic reasoning in language games.
\newblock \emph{Science}, 336\penalty0 (6084):\penalty0 998--998, 2012.

\bibitem[Golland et~al.(2010)Golland, Liang, and Klein]{golland2010game}
Dave Golland, Percy Liang, and Dan Klein.
\newblock A game-theoretic approach to generating spatial descriptions.
\newblock In \emph{Proceedings of the 2010 conference on empirical methods in
  natural language processing}, pages 410--419. Association for Computational
  Linguistics, 2010.

\bibitem[Goodman and Stuhlm{\"u}ller(2013)]{goodman2013knowledge}
Noah~D Goodman and Andreas Stuhlm{\"u}ller.
\newblock Knowledge and implicature: Modeling language understanding as social
  cognition.
\newblock \emph{Topics in cognitive science}, 5\penalty0 (1):\penalty0
  173--184, 2013.

\bibitem[Grice et~al.(1975)Grice, Cole, Morgan, et~al.]{grice1975logic}
H~Paul Grice, Peter Cole, Jerry Morgan, et~al.
\newblock Logic and conversation.
\newblock \emph{1975}, pages 41--58, 1975.

\bibitem[Gurvits(2003)]{gurvits2003classical}
Leonid Gurvits.
\newblock Classical deterministic complexity of edmonds' problem and quantum
  entanglement.
\newblock In \emph{Proceedings of the thirty-fifth annual ACM symposium on
  Theory of computing}, pages 10--19. ACM, 2003.

\bibitem[Gweon et~al.(2014)Gweon, Pelton, Konopka, and Schulz]{gweon2014sins}
Hyowon Gweon, Hannah Pelton, Jaclyn~A Konopka, and Laura~E Schulz.
\newblock Sins of omission: Children selectively explore when teachers are
  under-informative.
\newblock \emph{Cognition}, 132\penalty0 (3):\penalty0 335--341, 2014.

\bibitem[Hadfield-Menell et~al.(2016)Hadfield-Menell, Russell, Abbeel, and
  Dragan]{hadfield2016cooperative}
Dylan Hadfield-Menell, Stuart~J Russell, Pieter Abbeel, and Anca Dragan.
\newblock Cooperative inverse reinforcement learning.
\newblock In \emph{Advances in neural information processing systems}, pages
  3909--3917, 2016.

\bibitem[Ho et~al.(2016)Ho, Littman, MacGlashan, Cushman, and
  Austerweil]{ho2016showing}
Mark~K Ho, Michael Littman, James MacGlashan, Fiery Cushman, and Joseph~L
  Austerweil.
\newblock Showing versus doing: Teaching by demonstration.
\newblock In \emph{Advances in Neural Information Processing Systems}, pages
  3027--3035, 2016.

\bibitem[Ho et~al.(2018)Ho, Littman, Cushman, and
  Austerweil]{ho2018effectively}
Mark~K Ho, Michael~L Littman, Fiery Cushman, and Joseph~L Austerweil.
\newblock Effectively learning from pedagogical demonstrations.
\newblock In \emph{Proceedings of the Annual Conference of the Cognitive
  Science Society}, 2018.

\bibitem[Idel(2016)]{idel2016review}
Martin Idel.
\newblock A review of matrix scaling and sinkhorn's normal form for matrices
  and positive maps.
\newblock \emph{arXiv preprint arXiv:1609.06349}, 2016.

\bibitem[Kao et~al.(2014)Kao, Wu, Bergen, and Goodman]{Kao2014}
Justine~T Kao, Jean~Y Wu, Leon Bergen, and Noah~D Goodman.
\newblock Nonliteral understanding of number words.
\newblock \emph{Proceedings of the National Academy of Sciences}, 111\penalty0
  (33):\penalty0 12002--12007, 2014.

\bibitem[Knight(2008)]{knight2008sinkhorn}
Philip~A Knight.
\newblock The sinkhorn--knopp algorithm: convergence and applications.
\newblock \emph{SIAM Journal on Matrix Analysis and Applications}, 30\penalty0
  (1):\penalty0 261--275, 2008.

\bibitem[Lassiter and Goodman(2017)]{lassiter2017adjectival}
Daniel Lassiter and Noah~D Goodman.
\newblock Adjectival vagueness in a bayesian model of interpretation.
\newblock \emph{Synthese}, 194\penalty0 (10):\penalty0 3801--3836, 2017.

\bibitem[Menon and Schneider(1969)]{menon1969spectrum}
MV~Menon and Hans Schneider.
\newblock The spectrum of a nonlinear operator associated with a matrix.
\newblock \emph{Linear Algebra and its applications}, 2\penalty0 (3):\penalty0
  321--334, 1969.

\bibitem[Moon et~al.(2009)Moon, Gunther, and Kupin]{moon2009sinkhorn}
Todd~K Moon, Jacob~H Gunther, and Joseph~J Kupin.
\newblock Sinkhorn solves sudoku.
\newblock \emph{IEEE Transactions on Information Theory}, 55\penalty0
  (4):\penalty0 1741--1746, 2009.

\bibitem[Osborne(1960)]{osborne1960pre}
EE~Osborne.
\newblock On pre-conditioning of matrices.
\newblock \emph{Journal of the ACM (JACM)}, 7\penalty0 (4):\penalty0 338--345,
  1960.

\bibitem[Pretzel(1980)]{pretzel1980convergence}
Oliver Pretzel.
\newblock Convergence of the iterative scaling procedure for non-negative
  matrices.
\newblock \emph{Journal of the London Mathematical Society}, 2\penalty0
  (2):\penalty0 379--384, 1980.

\bibitem[Room(1938)]{room1938geometry}
Thomas~Gerald Room.
\newblock \emph{The geometry of determinantal loci}, volume~1.
\newblock The University Press, 1938.

\bibitem[Shafto and Goodman(2008)]{shafto2008teaching}
Patrick Shafto and Noah Goodman.
\newblock Teaching games: Statistical sampling assumptions for learning in
  pedagogical situations.
\newblock In \emph{Proceedings of the 30th annual conference of the Cognitive
  Science Society}, pages 1632--1637. Cognitive Science Society Austin, TX,
  2008.

\bibitem[Shafto et~al.(2012{\natexlab{a}})Shafto, Eaves, Navarro, and
  Perfors]{shafto2012epistemic}
Patrick Shafto, Baxter Eaves, Daniel~J Navarro, and Amy Perfors.
\newblock Epistemic trust: Modeling children’s reasoning about others’
  knowledge and intent.
\newblock \emph{Developmental science}, 15\penalty0 (3):\penalty0 436--447,
  2012{\natexlab{a}}.

\bibitem[Shafto et~al.(2012{\natexlab{b}})Shafto, Goodman, and
  Frank]{shafto2012learning}
Patrick Shafto, Noah~D Goodman, and Michael~C Frank.
\newblock Learning from others: The consequences of psychological reasoning for
  human learning.
\newblock \emph{Perspectives on Psychological Science}, 7\penalty0
  (4):\penalty0 341--351, 2012{\natexlab{b}}.

\bibitem[Shafto et~al.(2014)Shafto, Goodman, and Griffiths]{Shafto2014}
Patrick Shafto, Noah~D Goodman, and Thomas~L Griffiths.
\newblock A rational account of pedagogical reasoning: Teaching by, and
  learning from, examples.
\newblock \emph{Cognitive Psychology}, 71:\penalty0 55--89, 2014.

\bibitem[Shneidman et~al.(2016)Shneidman, Gweon, Schulz, and
  Woodward]{shneidman2016learning}
Laura Shneidman, Hyowon Gweon, Laura~E Schulz, and Amanda~L Woodward.
\newblock Learning from others and spontaneous exploration: A cross-cultural
  investigation.
\newblock \emph{Child Development}, 87\penalty0 (3):\penalty0 723--735, 2016.

\bibitem[Sinkhorn(1972)]{sinkhorn1972continuous}
Richard Sinkhorn.
\newblock Continuous dependence on a in the dad theorems.
\newblock \emph{Proceedings of the American Mathematical Society}, 32\penalty0
  (2):\penalty0 395--398, 1972.

\bibitem[Sinkhorn and Knopp(1967{\natexlab{a}})]{Sinkhorn1967}
Richard Sinkhorn and Paul Knopp.
\newblock Concerning nonnegative matrices and doubly stochastic matrices.
\newblock \emph{Pacific Journal of Mathematics}, 21\penalty0 (2):\penalty0
  343--348, 1967{\natexlab{a}}.

\bibitem[Sinkhorn and Knopp(1967{\natexlab{b}})]{sinkhorn1967concerning}
Richard Sinkhorn and Paul Knopp.
\newblock Concerning nonnegative matrices and doubly stochastic matrices.
\newblock \emph{Pacific Journal of Mathematics}, 21\penalty0 (2):\penalty0
  343--348, 1967{\natexlab{b}}.

\bibitem[Sinkhorn and Knopp(1969)]{sinkhorn1969problems}
Richard Sinkhorn and Paul Knopp.
\newblock Problems involving diagonal products in nonnegative matrices.
\newblock \emph{Transactions of the American Mathematical Society},
  136:\penalty0 67--75, 1969.

\bibitem[Soules(1991)]{soules1991rate}
George~W Soules.
\newblock The rate of convergence of sinkhorn balancing.
\newblock \emph{Linear algebra and its applications}, 150:\penalty0 3--40,
  1991.

\bibitem[Tomasello(2009)]{tomasello2009we}
Michael Tomasello.
\newblock \emph{Why we cooperate}.
\newblock MIT Press, Cambridge, MA, 2009.

\bibitem[Tverberg(1976)]{tverberg1976sinkhorn}
Helge Tverberg.
\newblock On sinkhorn's representation of nonnegative matrices.
\newblock \emph{Journal of Mathematical Analysis and Applications}, 54\penalty0
  (3):\penalty0 674--677, 1976.

\bibitem[Villani(2008)]{villani2008optimal}
C{\'e}dric Villani.
\newblock \emph{Optimal transport: old and new}, volume 338.
\newblock Springer Science \& Business Media, 2008.

\bibitem[Vong et~al.(2018)Vong, Sojitra, Reyes, Yang, and Shafto]{vongbayesian}
Wai~Keen Vong, Ravi~B Sojitra, Anderson Reyes, Scott Cheng-Hsin Yang, and
  Patrick Shafto.
\newblock Bayesian teaching of image categories.
\newblock In \emph{Proceedings of the Annual Conference of the Cognitive
  Science Society}, 2018.

\bibitem[Yang and Shafto(2017)]{yang2017explainable}
Scott Cheng-Hsin Yang and Patrick Shafto.
\newblock Explainable artificial intelligence via bayesian teaching.
\newblock In \emph{NIPS 2017 workshop on Teaching Machines, Robots, and
  Humans}, 2017.

\bibitem[Yang et~al.(2018)Yang, Yu, Givchi, Wang, Vong, and
  Shafto]{YangYGWVS18}
Scott~Cheng{-}Hsin Yang, Yue Yu, Arash Givchi, Pei Wang, Wai~Keen Vong, and
  Patrick Shafto.
\newblock Optimal cooperative inference.
\newblock In \emph{{AISTATS}}, volume~84 of \emph{Proceedings of Machine
  Learning Research}, pages 376--385. {PMLR}, 2018.

\end{thebibliography}

\newpage

\appendix
\section{Supplemental Material}\label{sec:Supp}

\begin{proof}[\textbf{Proof of Proposition~\ref{prop:stable_format}}]
Let the column sums of $P$ be $\mathcal{C}=\{c_1, \dots, c_v\}$ and row sums of $Q$ be $\mathcal{R}=\{r_1, \dots, r_u\}$.
$(P,Q)$ is a SK stable pair implies that column normalization of $P$ equals $Q$, i.e. if $q_{ij}>0$ then $q_{ij}=p_{ij}/c_j$, and further row normalization of $Q$ equals $P$, i.e. if $p_{ij}>0$ then $p_{ij}=q_{ij}/r_i$.
So we have that $p_{ij}=p_{ij}/(r_i\cdot c_j)$ $\implies r_i\cdot c_j=1$ (Claim~$(*)$).
Therefore $c_j\in \mathcal{C}$ $\implies$ $1/c_j\in \mathcal{R}$ and  $r_i\in \mathcal{R}$ $\implies$ $1/r_i\in \mathcal{C}$. In particular, let $c_{max}=\max\{c_1, \dots, c_v\}$ and $r_{min}=\min\{r_1, \dots, r_u\}$. 
We have that $c_{max}=1/r_{min}$.

With permutation, we may assume that the columns with sum $c_{max}$ in $P$ are the first $v_1$ columns and
the rows with sum $r_{min}$ in $Q$ are the first $u_1$ rows. Note that an element $p_{ij}$ in the first $v_1$ columns of $P$ is positive only if it is in the first $u_1$ rows. Otherwise assume that $p_{i_0j}>0$ for $i_0>u_1$, then Claim~$(*)$ implies that $c_j\cdot r_{i_0}=1$ $\implies$ $r_{i_0}=1/c_j=1/c_{max}=r_{min}$. This contradicts to $i_0>u_1$.
Similarly, we may show that an element $q_{ij}$ in the first $u_1$ rows of $Q$ is positive only if it is in the first $v_1$ columns. Further note that $Q$ and $P$ have the same pattern. So we have that for $i\leq u_1$, $p_{ij}>0$ only if $j\leq v_1$.
Therefore let $B_1$ be the submatrix of $P$ formed by the first $u_1$ rows and first $v_1$ columns and $P_1$ ($Q_1$) be the submatrix of $P$ ($Q$) formed by the last $u-u_1$ rows and the last $v-v_1$ columns. We just showed that $P=\diag\{B_1, P_1\}$ and the column sum $c_{max}$ of $B_1$ is a constant (equals $u_1/v_1$). $(P_1, Q_1)$ is a \textit{SK stable} pair with smaller dimension. Hence, inductively, the proposition holds. 
\end{proof}


\begin{lemma}\label{limit_intermediate}
If a pair of matrices $(P, Q)$ as in Proposition \ref{EU} exists, the pattern of any pair of limit matrices $(\LLmat^\prime, \TTmat^\prime$)  is intermediate between the pattern of $(P, Q)$ and the pattern of $\M$, namely, $(P, Q) \prec (\LLmat^\prime, \TTmat^\prime) \prec \M$.
\end{lemma} 
\begin{proof}
Let the dimension of $\M$ be $u\times v $.
Denote the sequence of matrices generated by \textit{SK iteration} by $\{\LLmat^n, \TTmat^n\}
(n>0)$, where $\LLmat^n$ are row normalized and has column sums $\{c_{jn}\}_j$, and $\TTmat^n
$ are column normalized and has row sums $\{r_{in}\}_i$. 
As explained in \cite{pretzel1980convergence},
there exist diagonal matrices $X_n$ and $Y_n$ such that $\LLmat^n=X_n\M Y_n$ and $\TTmat^n=X_n\M Y_{n+1}$.
In particular, $X_n=\diag\{x_{1n}, \dots, x_{un}\}$ and $Y_n=\diag\{y_{1n}, \dots, y_{vn}\}$, where each 
$x_{in}$ is the product of row normalizing constants (reciprocal of row sums) of row-$i$ from step $1$ to $n$ and 
each $y_{jn}$ is the product of column normalizing constants (reciprocal of column sums) of column-$j$ from step $1$ to $n$.
Here, $Y_1$ is the identity matrix. 

Denote the row sums of $Q$ by $\{r_i\}^{u}_{i=1}$
and the column sums of $P$ by $\{c_j\}^v_{j=1}$.
Consider the following functions, we will show that they form an increasing sequence(the use of it will be clear later).
$$f_n=\prod\limits^{u}_{i=1}x_{in}^{1+\alpha r_i}\prod\limits^{v}_{j=1}y_{jn}^{\alpha +c_j},$$
$$g_n=\prod\limits^u_{i=1}x_{in}^{1+\alpha r_i}\prod\limits^v_{j=1}y_{jn+1}^{\alpha+c_j},$$ 

where, $\alpha\geq -\dfrac{\log s}{\log r},$ with $s=\frac{1}{\prod\limits_{j}c_j^{c_j}}$ and $r=\left(\frac{v}{u}\right)^v$.
\begin{align}\label{eq:fg}
\dfrac{g_n}{f_n}
=\prod\limits^v_{j=1}\left(\frac{y_{jn+1}}{y_{jn}} \right)^{\alpha +c_j}
&=\prod\limits^v_{j=1}\left(\frac{1}{c_{jn}} \right)^{\alpha +c_j}\nonumber \\
&\geq \prod\limits^v_{j=1}\left(\frac{1}{c_{jn}} \right)^{c_j}
\prod\limits^v_{j=1}\left(\frac{1}{c_{jn}} \right)^{\alpha} \,.
\end{align}

Due to Lemma 1 of \cite{berry2007algorithms} we have, $\frac{1}{\prod\limits_{j}c_{jn}^{c_j}}\geq \frac{1}{\prod\limits_{j}c_j^{c_j}}=s$, i.e. the first product of the right hand side of Inequality~\eqref{eq:fg} is greater or equal to ~$s$.

Moreover, by arithmetic and geometric means inequality, $\left(\prod\limits_{j=1}^v c_{jn}\right)^\frac{1}{v}\leq \dfrac{\sum\limits^v_{j=1} c_{jn}}{v}=\dfrac{u}{v}.$ 
Hence,
$\prod\limits^v_{j=1}\dfrac{1}{c_{jn}} =r$.
Therefore $\dfrac{g_n}{f_n}\geq s r^{\alpha}\geq 1$, where the second inequality holds because of the choice of $\alpha.$
Hence we have $g_n\geq f_n.$
The analogous argument holds for $f_{n+1}/g_n$.
So, we have $f_{n+1}\geq g_n\geq f_n$ (Claim~$*$). 

Now recall that $\LLmat^n=X_n\M Y_n$. In particular, we have $l^n_{ij}=x_{in}m_{ij}y_{jn}$, for $m_{ij}\neq0$. 
So $x_{in}y_{jn}=\frac{l^n_{ij}}{m_{ij}}$ and it is bounded above because
the elements $l^n_{ij}$ are bounded above by 1. One possible upper bound is $K=\frac{1}{ \min m_{ij}}$, where min is taken over non zero elements in $\M$. 

Moreover, let $d_{ij}= p_{ij}+\alpha \,q_{ij}$, then 
\begin{align*}
\prod \limits_{ij}\left(x_{in}y_{jn}\right)^{d_{ij}} &=
\prod \limits_{ij}\left(x_{in}y_{jn}\right)^{p_{ij}+\alpha \,q_{ij}} \\
&= \prod \limits_{i}x_{in}^{\sum_{j}p_{ij}+\alpha \,q_{ij}}
\prod \limits_{j}y_{jn}^{\sum_{i}p_{ij}+\alpha \,q_{ij}}\\
&= \prod \limits_{i}x_{in}^{1+\alpha\cdot r_i}\prod \limits_{j}y_{jn}^{\alpha+c_j}=f_n.
\end{align*}

Furthermore, if $d_{ij}\neq 0$, then $p_{ij}\neq 0 \implies m_{ij}\neq 0$ and hence $x_{in}y_{jn}\leq K$. 
Therefore $f_n\leq K^d,$ where $d=\sum\limits_{ij}d_{ij}$.
Together with Claim~$*$, we have
$\left(x_{in}y_{jn}\right)^{d_{ij}}K^{(d-d_{ij})}\geq f_n\geq f_1.$ So if $d_{ij}\neq 0$, then $x_{in}y_{jn}$ is bounded away from zero. Thus it follows that $l^n_{ij}$ is bounded away from zero for all $n$. Therefore $P\prec \LLmat^\prime$, where $\LLmat^\prime$ is the limit of a subsequence of $\LLmat^n$. 
Finally, since SK iteration perseveres zero elements, $\LLmat'\prec \M$.
Together, we have $P\prec \LLmat' \prec \M$. 
A similar argument holds for $Q$ and $\TTmat'$.
Thus the lemma holds.


\end{proof}


\begin{remark}\label{same_pattern}
Notice that, the choice of $(P, Q)$ is free within the constraints (having partial pattern of $\M$ and being \textit{SK stable}).  In particular, such matrix pairs can be partially ordered with respect to their patterns, and $(P, Q)$ can be selected such that they have the maximum possible pattern. Since the pattern of limit matrices must be intermediate between the pattern of $(P, Q)$ and the pattern of $\M$, it follows that, all the pairs of limit matrices must have the same pattern, which must be the maximum possible.  
\end{remark}

\begin{lemma}\label{lemma:diagonal_eq}
Any limit matrix of \textit{SK iteration} on $\M$
is diagonally equivalent to $\overbar {\M}$. 
\end{lemma}
\begin{proof}
Let $\LLmat'$ be the limit of the sequence $X_n^\prime \M Y_n^\prime$ (where the $\prime$ signifies any sub-sequence of \textit{SK iteration}). Then $\LLmat'$ is also the limit of the sequence $X_n^\prime \overbar {\M} Y_n^\prime$, 
since both $\overbar{\M}$ and $\LLmat'$ has the same pattern. 
In this case, Lemma 2 of \cite{pretzel1980convergence}
implies that there exist diagonal matrices $X, Y$ such that $\LLmat'=X\overbar{\M}Y$, i.e. $\LLmat'$ and $\overbar{\M}$ are diagonally equivalent.
\end{proof}

Proposition~1 in \cite{pretzel1980convergence} shows that: 
\begin{lemma}\label{lemma: diag_eq}
Let $A$ and $B$ be two matrices with the same row and columns sums. 
If there exists diagonal matrices $X$ and $Y$ such that $A=XBY$, then
$A=B$.
\end{lemma}

\begin{proof} [\textbf{Proof of Proposition~\ref{prop:exist_sk_subpattern}}]
We claim that: for a given $u\times v$ matrix $\M$, one may construct a \textit{binary} matrix $A \prec \M$ such that up to permutation, $A$ is block-wise diagonal of the form $\diag(B_1, \dots, B_k)$ where each $B_i$ is a row or column vector of ones (i.e. $B_i=(1, \dots, 1)$ or $B_i=(1, \dots, 1)^T$).
Let $P,Q$ be the row and column normalizations of $A$ respectively. 
It is straightforward to check that $(P,Q)$ is \textit{SK stable} and $P\prec \M$. Therefore, we only need to prove the claim.

We will prove the claim inductively on the dimension of $\M$. 
Let $n=\max\{u,v\}$.
When $n=1$, $\M$ is an $1\times 1$ matrix and the claim holds.
Now assume that the claim holds when $n\leq k-1$, we will show that
for a $u\times v$ matrix $\M$ with $\max\{u,v\}=k$, the claim still holds.
Without loss, we may assume that $v=k$. There are now two cases.

\textit{Case 1} When $u<k$, let $\M'$ be the sub-matrix formed by the first $k-1$ columns of $\M$.
Then $\M'$ is a $u\times (k-1)$ matrix. 
$\M$ has no zero columns implies that $\M'$ has no zero columns.
(1) If $\M'$ contains no zero rows, according to the inductive assumption, there exists 
a binary matrix $A' \prec \M'$ having the desired form. 
Note that the last column of 
$\M$ contains a non-zero element $m_{tk}$. Let $\mathbf{v}=(v_1, \dots, v_u)^T$ be the column vector with $v_i=0$ if $i\neq t$ and $v_t=1$. The desire $A\prec \M$ is then constructed using $A'$ and $\mathbf{v}$ as following.
The $t$-th row of $A'$ must have a non-zero element $a'_{ts}$ (as $A'$ has no zero row). Denote the block contains $a'_{ts}$ by $B'$.
If $B'$ is a row vector with all ones, then $A$ is obtained by augmenting $A'$ by $\mathbf{v}$, i.e. $A=[A', \mathbf{v}]$. 
Otherwise, we may replace $a'_{s,t}$ in $A'$ by zero and denote the resulting matrix by $A^{''}$. 
$A$ is obtained by augmenting $A^{''}$ by $\mathbf{v}$, i.e. $A=[A^{''}, \mathbf{v}]$.  

(2) If $\M'$ contains zero rows, let $\M^*$ be the matrix obtained from $\M'$ by omitting rows with indices in $S_{zero}$ where $S_{zero}$ is the index set of zero rows. Then there exists $A^*\prec \M^*$ according to the inductive assumption. Let $A'\prec \M'$ be the matrix obtained from $A^*$ by inserting back the zero rows (at indices $S_{zero}$).
Note that $\M$ contains no zero rows implies that $m_{ik}>0$ for any $i\in S_{zero}$. Let $\mathbf{v}=(v_1, \dots, v_u)^T$ be the column vector where $v_i=1$ if $i\in S_{zero}$ and $v_i=0$ otherwise. $A$ is obtained by augmenting $A'$ by $\mathbf{v}$, i.e. $A=[A', \mathbf{v}]$.

\textit{Case 2} When $u=k$, let $\M'$ be the sub-matrix formed by the first $k-1$ rows $\M$. Then depending whether $M'$ contains zero column, one may construct $A$ as in case~1. In all circumstances, it is easy to check that the defined $A$ has the desired format by construction. Hence claim also holds for any matrix $\M$ (or its transpose) of the form $u\times k$.
\end{proof}

\begin{proof}[\textbf{Proof of Corollary~\ref{cor:M_barM_SKsame}}]
 Let $(\LLmat,\TTmat)$ and $(\overbar \LLmat,\overbar \TTmat)$ be the limit of SK iteration on $\M$ and $\overbar \M$ respectively. It is enough to show that $\LLmat=\overbar \LLmat$.
Lemma \ref{lemma:diagonal_eq} implies that both $\overbar \LLmat$ and $\LLmat$ are diagonally equivalent to $\overbar{\M}$. Therefore, $\LLmat$ is diagonally equivalent to $\overbar \LLmat$. Further since both $\LLmat$ and $\overbar \LLmat$ have the same pattern as $\overbar \M$, Proposition~\ref{prop:stable_format} shows that they have the same row and column sums. Hence, Lemma \ref{lemma: diag_eq} implies that $\LLmat=\overbar \LLmat$. 
\end{proof}

\begin{proof}[\textbf{Proof of Proposition~\ref{prop:preimage}}]
Since $\M$ and $\overbar{\M}$ have exactly the same positive diagonals, we may assume that $\M$ has total support.
Suppose that $\M \in \Phi^{-1}(\LLmat)$, i.e. $\Phi(\M)=\LLmat$. Since $\M$ has total support, \cite{sinkhorn1967concerning} implies that 
there exists diagonal matrices $X=\diag(x_1, \dots, x_n)$ and $Y=\diag(y_1, \dots, y_n)$ such that $\M=X \LLmat Y$.
In particular, $m_{ij}=x_i\times l_{ij} \times y_j$ holds, for any element $m_{ij}$. 
Let $D^{\M}_1=\{m_{i,\sigma(i)}\}$, $D^{\M}_2=\{m_{i,\sigma'(i)}\}$
be two positive diagonals of $\M$ and $D^{\LLmat}_1=\{l_{i,\sigma(i)}\}$, $D^{\LLmat}_2=\{l_{i,\sigma'(i)}\}$ be the corresponding positive diagonals in $\LLmat$. 
Then: 
\begin{equation}\label{eq:cross_ratio}
\footnotesize
 \begin{split}
CR(D^{\M}_1,D^{\M}_2)& =\frac{\Pi_{i=1}^{n}m_{i,\sigma(i)}}{\Pi_{i=1}^{n}m_{i,\sigma'(i)}}
=\frac{\Pi_{i=1}^{n}x_i\times l_{i,\sigma(i)} \times y_{\sigma(i)}}{\Pi_{i=1}^{n} x_i\times l_{i,\sigma'(i)}\times y_{\sigma'(i)}}\\
& = \frac{\Pi_{i=1}^{n}x_i\times \Pi_{i=1}^{n} l_{i,\sigma(i)} \times \Pi_{i=1}^{n} y_{\sigma(i)}}{\Pi_{i=1}^{n} x_i\times \Pi_{i=1}^{n} l_{i,\sigma'(i)}\times \Pi_{i=1}^{n} y_{\sigma'(i)}}\\
&=\frac{\Pi_{i=1}^{n}l_{i,\sigma(i)}}{\Pi_{i=1}^{n}l_{i,\sigma'(i)}}=CR(D^{\LLmat}_1,D^{\LLmat}_2)
\end{split}  
\end{equation}
We have established the `if' direction. 
Now, for the `only if' direction,  suppose that $\M\crs \LLmat$. Let $\M^*=\Phi(\M)$. Then $\M^*\in \B$ and $\M\in \Phi^{-1}(\M^*)$.
According to Equation~\eqref{eq:cross_ratio}, $\M^*\crs\M$ and so $\M^*\crs\LLmat$. Let $k=d_1^{\M^*}/d_1^{\LLmat}>0$, 
where $D_1^{\M^*}$ and $D_1^{\LLmat}$ are positive diagonals determined by the same $\sigma\in S_n$.
$\M^*\crs\LLmat$ implies that $\Pi_{i=1}^{n}m^*_{i,\alpha(i)}=k\times \Pi_{i=1}^{n}l_{i,\alpha(i)}$ for any $\alpha \in S_n$.
Note that \textit{distinct} doubly stochastic matrices do not have proportional corresponding diagonal products 
(For a proof see \cite{sinkhorn1969problems}). Hence, $\LLmat=\M^*=\Phi(\M)$.
\end{proof}

\begin{thm}[Birkhoff-von Neumann theorem](\cite{dufosse2016notes})\label{thm:BN}
For any $n\times n$ doubly stochastic matrix $A$, there exist $\theta_i \geq 0$ with $\sum_{i=1}^{k}\theta_{i}=1$ and permutation matrices $\{P_1,\dots,P_k\}$  such that $A=\sum_{i=1}^{k} \theta_iP_i$. This representation is also called Birkhoff-von Neumann (BvN) decomposition of $A$.
\end{thm}

\begin{proof}[\textbf{Proof of Proposition~\ref{prop:dist_preimage}}]
 $\M^1$ has total support implies that there exist two diagonal matrices $X=\diag\{x_1, \dots, x_n\}$ and $Y=\diag\{y_1, \dots, y_n\}$ such that $\M^1=X\LLmat^1 Y$.
Let $\M^2=X\LLmat^2 Y$ and $C=\max_{ij}\{x_iy_j\}$. Then $d(\LLmat^1, \LLmat^2) \leq \epsilon$ $\implies$ $|l^1_{ij}-l^2_{ij}|\leq \e$ $\implies$ $|x_i\cdot l^1_{ij}\cdot y_j-x_i\cdot l^2_{ij}\cdot y_j|\leq C\e$ $\implies$ $|m^1_{ij}-m^2_{ij}|\leq C\e$. Thus, $d(\M^1, \M^2)\leq C\e$.
\end{proof}

\textbf{Construction of Homeomorphic $\Phi$.}

As mentioned above, for any $\M\in \overbar{\A}$, there exist two diagonal matrices $X$ and $Y$ such that $\M=X\Phi(\M)Y$. Note the choice of $X$ and $Y$ is unique only up to a scalar. 
This can be made deterministic by requiring the last positive element of $Y$ to be $1$, i.e. $y_n=1$. In this way $\Phi$ can be viewed as a map : $\overbar{\A} \to \R^{2n-1}_{+}\times \B$ where $\M \mapsto [(x_1, \dots, x_n, y_1, \dots, y_{n-1}), \Phi(\M)]$. \cite{tverberg1976sinkhorn} showed that:

\begin{prop} 
$\Phi:\overbar{\A}\to \R^{2n-1}_{+}\times \B$ is continuous, invertible and the inverse is also continuous. Thus $\Phi$ is homeomorpic.
\end{prop}

\textbf{Role of zeros}: Results derived in Section~\ref{sec: rectangular}-\ref{sec:cr_sn} do not depend on zero elements. The lower bounds of CI derived in Section~\ref{sec:bounds} are closely related to the amount of zeros and their locations. However, together with the sensitivity analysis, these bounds can still be used as an approximation for lower bounds of CI for matrices with very small elements (instead of exact zeros). In some models, zero elements could appear. For instance, in linguistic applications, if an utterance is not consistent with a referent, its corresponding element would be zero \citep{golland2010game}. More generally, it is an interesting question as to whether more models should assign zero probability to some possible outcomes. As noted, most probabilistic models do not currently assign zero probability to any outcomes. However, if one, for example, wants models that are explainable via examples our results show that assigning zero probability to some outcomes is a desirable feature. Beyond explainability, as far as we know, there are no principled reason for not assigning zero probability to some possible outcomes. Finally, having zeros reduces the number of positive diagonals, which is a special case of the more general problem of establishing bounds based on cross-ratio (Corollary~\ref{cor:total_support}). This is considerably more challenging and a direction for future work.

\textbf{Other connections.} Sinkhorn iteration finds its way in many other applications in variety of fields. To name a few: transportation planning to predict flow in a traffic network \citep{fienberg1970iterative}, contingency table analysis which has many uses in biology, economics etc. \citep{fienberg1970iterative}, decreasing condition numbers which is of importance in numerical analysis \citep{osborne1960pre}. Moreover, there are many algorithms implemented as generalizations of Sinkhorn matrix balancing to solve problems such as Edmonds problem \citep{gurvits2003classical}, Sudoku Solvers \citep{moon2009sinkhorn} and web page ranking algorithms \citep{knight2008sinkhorn}. More applications and a comprehensive discussion can be found in \citep{idel2016review} and references therein.

\end{document}